\documentclass{article}
\usepackage{xspace}
\usepackage{txie}

\def\shownotes{0}  
\ifnum\shownotes=1
\newcommand{\authnote}[2]{\noindent$^{\text{\fontfamily{cmtt}\em #1:}}\langle${\sf\small #2}$\rangle$}
\else
\newcommand{\authnote}[2]{}
\fi
\usepackage{microtype}
\usepackage{booktabs} 

\usepackage[accepted]{icml2021}

\icmltitlerunning{Interaction-Grounded Learning}

\begin{document}

\twocolumn[
\icmltitle{Interaction-Grounded Learning}



\icmlsetsymbol{equal}{*}

\begin{icmlauthorlist}
\icmlauthor{Tengyang Xie}{uiuc}
\icmlauthor{John Langford}{msr}
\icmlauthor{Paul Mineiro}{msr}
\icmlauthor{Ida Momennejad}{msr}
\end{icmlauthorlist}

\icmlaffiliation{uiuc}{University of Illinois at Urbana-Champaign}
\icmlaffiliation{msr}{Microsoft Research, New York City}

\icmlcorrespondingauthor{Tengyang Xie}{tx10@illinois.edu}
\icmlcorrespondingauthor{John Langford}{jcl@microsoft.com}
\icmlcorrespondingauthor{Paul Mineiro}{pmineiro@microsoft.com}
\icmlcorrespondingauthor{Ida Momennejad}{idamo@microsoft.com}

\icmlkeywords{Machine Learning, ICML}

\vskip 0.3in
]



\printAffiliationsAndNotice{}  



\begin{abstract}


Consider a prosthetic arm, learning to adapt to its user's control signals. We propose \emph{Interaction-Grounded Learning} for this novel setting, in which a learner's goal is to interact with the environment with no grounding or explicit reward to optimize its policies. Such a problem evades common RL solutions which require an explicit reward. The learning agent observes a multidimensional \emph{context vector}, takes an \emph{action}, and then observes a multidimensional \emph{feedback vector}. This multidimensional feedback vector has \emph{no} explicit reward information. In order to succeed, the algorithm must learn how to evaluate the feedback vector to discover a latent reward signal, with which it can ground its policies without supervision. We show that in an Interaction-Grounded Learning setting, with certain natural assumptions, a learner can discover the latent reward and ground its policy for successful interaction. We provide theoretical guarantees and a proof-of-concept empirical evaluation to demonstrate the effectiveness of our proposed approach.

%
%
%

\end{abstract}


\section{Introduction}
We consider a novel setting. A learner's goal is to interact with an environment, and while the environment reacts to the learner's actions, its feedback does not provide an explicit reward signal. Because the learner must deduce a grounding for the feedback solely via interaction, we call this setting \emph{Interaction-Grounded Learning} (IGL).

There are \emph{many} examples and potential applications of Interaction-Grounded Learning.  In a visual domain, a robot could learn to interact effectively with a user's personalized hand gestures.  In an audio domain, a smart speaker could learn to give useful responses based upon a user's idiosyncratic ways of signaling pleasure or annoyance. In a BCI (Brain Computer Interface) setting, a computer could learn to interact with a human based upon a user's EEG signals.

These problems are not easily solved using traditional reinforcement learning (RL), inverse RL, or supervised learning because the \emph{absence of an explicit reward} is an essential ambiguity of the setting.  When solving these problems, a key issue is agreeing on a shared code between a human and a computer.  For example, in neurofeedback and BCI an algorithm is often trained via supervised learning techniques to interpret brain signals, which are in turn used to interact with or train human participants ~\cite{BCI_arm, closedloop2015, deBettencourt2015, neurofeedback2020, Ireti, Ivan}.  The supervised learning techniques used, however, tend to be laborious and can require chronic retraining.  Another challenge of BCI solutions is that, over time, the initial placement of sensors that read user signals, as well as the interpretation of the signals, often change and require re-calibration. IGL opens up the possibility of more natural and continual self-calibration.

\begin{figure}[thb]
    \hskip -0.1in
    \includegraphics[width=1.2\columnwidth]{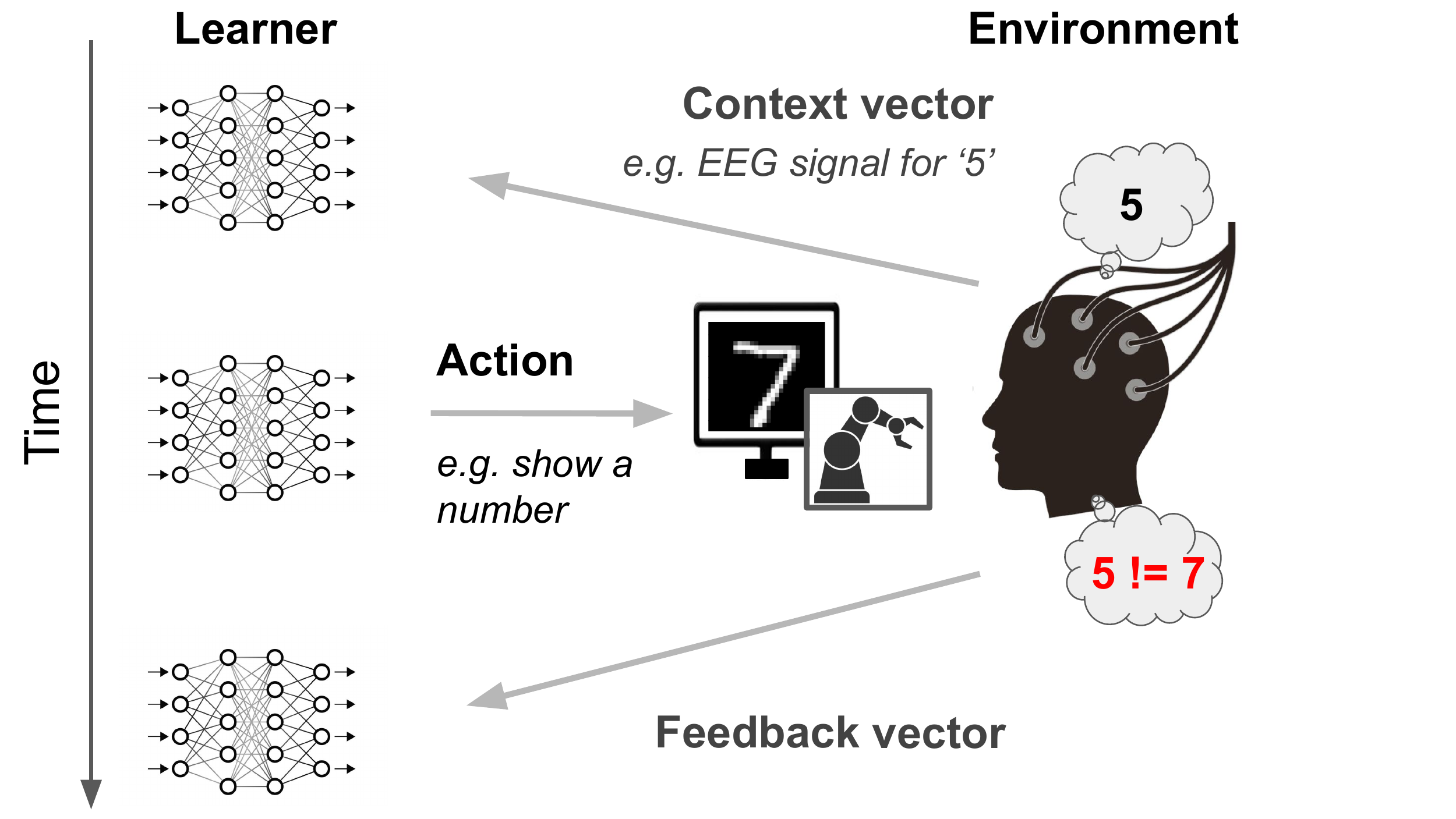}
    \caption{A schematic example of the Interaction-Grounded Learning (IGL) setting. The learner observes a context  vector (e.g. EEG signal of interacting partner thinking about number 5 or intention to grab a cup), takes an action (e.g. show number 7 or move arm), and observes a feedback vector (e.g. an ungrounded multidimensional EEG signal). Importantly, there is \emph{no explicit reward} signal. The learner assumes there is a latent reward in the feedback vector then learns a reward decoder and an optimal action policy given those assumptions.}
    \vspace{-0.3in}
    \label{fig:IGL_setting}
\end{figure}

\begin{figure*}[htb]
    \centering
    \begin{subfigure}[b]{0.32\linewidth}
        \centering
        \includegraphics[width=0.625\linewidth,page=1]{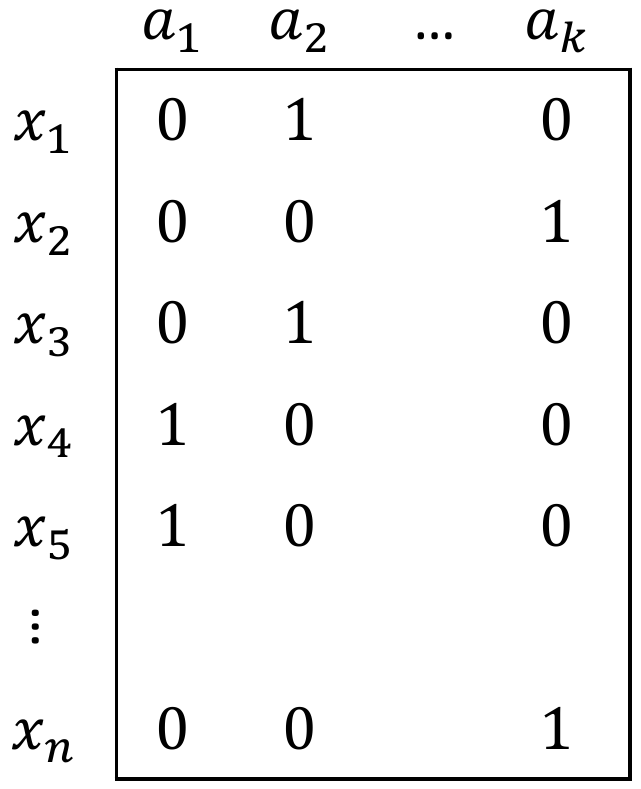}
        \caption{Supervised Classification}
        \label{fig:supln}
    \end{subfigure}
    ~
    \begin{subfigure}[b]{0.32\linewidth}
        \centering
        \includegraphics[width=0.625\linewidth,page=2]{figures/bci.pdf}
        \caption{Contextual Bandits (CBs)}
        \label{fig:cbs}
    \end{subfigure}
    ~
    \begin{subfigure}[b]{0.32\linewidth}
        \centering
        \includegraphics[width=0.625\linewidth,page=3]{figures/bci.pdf}
        \caption{Interaction-Grounded Learning (IGL)}
        \label{fig:bci}
    \end{subfigure}
    \caption{An example of different learning approaches. {\bf Figure \ref{fig:supln}:} Supervised learning assumes the full reward information is given for each context. {\bf Figure \ref{fig:cbs}:} Contextual bandits gives exact reward information on the selected action. {\bf Figure \ref{fig:bci}:} In Interaction-Grounded Learning, a feedback vector is observed instead of a reward.}
    \label{fig:diff_setting}
\end{figure*}

In the IGL setting, the learner observes a multidimensional \emph{context vector}, takes an \emph{action}, and then observes a multidimensional \emph{feedback vector}. This feedback vector has \emph{no} explicit reward information, but does carry information about a latent reward.   In order to succeed, the learning algorithm must discover a good grounding for the feedback vector suitable for evaluating interactive policies. Interaction-Grounded Learning can then empower the agent to interpret multi-dimensional user signals in terms of latent reward, and optimize its behaviour using this inferred reward.  While this may appear impossible at first, we prove that Interaction-Grounded Learning can succeed when three assumptions hold: (I) the feedback vector has information about the latent reward, (II) the feedback vector is conditionally independent of the action and context given the reward, and (III) random actions have low expected reward. 

After introducing the Interaction-Grounded Learning setting in section~\ref{sec:setting} with more detail, we propose a potential algorithm to solve IGL: Explore-Exploit Ground learning, or \algoname. The \algoname\ learner takes random actions during gradually increasing exploration epoch. Later during intermittent exploitation, \algoname\  grounds the reward in its interaction history during the exploration epochs. We prove that \algoname\ can solve IGL under the assumptions mentioned above.

In section~\ref{sec:discussion} we further discuss these assumptions, their applicability, and potential relaxation.

\subsection*{Our contribution}

We define the Interaction-Grounded Learning setting in section~\ref{sec:setting}. Given the scope of potential applications, we believe this setting may form a core area of study in the future. 

Section~\ref{sec:LfCI} studies the feasibility of IGL in a simplest-possible batch setting, clarifying the assumptions under which it is tractable and providing a proof that it is indeed possible.

The batch setting appears unnatural in most IGL applications where a more online approach is called for. Therefore, in section~\ref{sec:IA} we present \algoname, an algorithm for online Interaction-Grounded Learning. We prove that \algoname succeeds under similar assumptions to the simple batch setting.  In section~\ref{sec:experiments} we conduct proof-of-concept experiments showing IGL is possible in both the batch and online cases.

We then consider an alternatives to IGL, i.e., an unsupervised learning approach to extracting rewards. In section~\ref{sec:failunsup} we show a scenario, in which unsupervised learning cannot succeed without additional assumptions.  The key insight is that the distribution of feedback vectors has multiple natural clusterings, which correspond to the solutions of different IGL problems.  Restated, any unsupervised learning approach that can succeed on one instance of a problem must fail on another instance of the problem, while the IGL approaches we discuss here can succeed on both.

\section{Problem Definition}
\label{sec:setting}

We propose and analyze the \emph{Interaction-Grounded Learning} setting, in which the learner uses interaction to create a grounding for evaluation and optimization of a feedback vector.  Each round, the stationary environment generates an i.i.d. context $x \in \Xcal$ from a distribution $d_0$ and reveals it to the learner, which chooses an action $a \in \Acal$ from a finite action set ($|\Acal| = K$); the environment then generates an unobserved binary reward $r \in \{0, 1\}$ and a feedback vector $y \in \Ycal$ conditional on $(x, a)$, and reveals $y$ to the learner.  The reward can be either deterministic or stochastic, and we denote $R(x,a) \coloneqq \E[r|x,a]$.  In this setting, the spaces of both context $\Xcal$ and feedback vector $\Ycal$ can be uncountably rich.

We use $\pi \in \Pi: \Xcal \to \Delta(\Acal)$ to denote a (stochastic) policy, and we define the expected return of policy $\pi$ as $V(\pi) \coloneqq \E_{(x,a) \sim d_0 \times \pi}[R(x,a)]$.  Learning aims to achieve low regret with respect to the optimal policy in the policy class $\pi$,
\begin{align}
\label{eq:pistar}
\pi^\star \coloneqq &~ \argmax_{\pi \in \Pi} V(\pi),
\end{align}
while interacting with the environment only through observations of  context-action-feedback $(x,a,y)$ triples.


The Interaction-Grounded Learning setting extends the border of interactive machine learning. Figure~\ref{fig:diff_setting} is an example comparison of supervised learning, contextual bandits, and Interaction-Grounded Learning. In this case, an image of exact reward is provided as feedback vector $y$.  This example does not exhaust the expressiveness of our generative model, which allows $(r, y)$ to be drawn jointly conditional on $(x, a)$.  However, in Section~\ref{sec:LfCI} we impose further assumptions to make progress, and Figure~\ref{fig:diff_setting} is consistent with those assumptions.




\section{Learning with Conditional Independence}
\label{sec:LfCI}

Direct empirical estimation of $V(\pi)$ requires access to unobserved information, frustrating the application of traditional techniques.  Our algorithms instead employ a reward decoder class $\psi \in \Psi: \Ycal \to [0, 1]$.  Treating $\psi(y)$ as an approximation to $\mathbb{E}\left[r | y\right]$ motivates the decoded reward $V(\pi, \psi) \doteq \E_{(x,a) \sim d_0 \times \pi}[\psi(y)]$.  Our algorithms jointly choose $\psi$ and $\pi$ to maximize the decoded reward. The main challenge of learning the best $\pi$ and $\psi$ jointly over $\Pi$ and $\Psi$ is that the reward $r$ is unobserved.  Under what conditions does maximizing the decoded reward ensure low regret with respect to the unobserved latent reward?  We leverage the following assumption to enable success.

\begin{assumption}[Conditional Independence]
\label{asm:condind}
For arbitrary $(x,a,r,y)$ tuple where $r$ and $y$ are generated based on a context $x$ and action $a$, we assume the feedback vector $y$ is conditionally independent of action $a$ and context $x$ given latent reward $r$.  In other words we assume that $x,a \indep y | r$.
\end{assumption}

Assumption \ref{asm:condind} ensures that the feedback vector $y$ is generated only based on the latent reward $r$, without further dependence on the action $a$ or the context $x$.  As discussed in section~\ref{sec:discussion}, this assumption is reasonable for some problems. While the assumption may seem unreasonable for others, it could be satisfied by existing orthogonalization practices in BCI (which are applied prior to applying machine learning) and could conceivably be relaxed in future work. Informally this assumption enables progress on the learning objective by ensuring that the mistakes of a reward predictor have a uniform effect across the policy class. This allows the decoded latent reward to be a faithful representation of the expected return of a policy.  We carefully elaborate this argument below.

\subsection{Proxy Learning Objective}
\label{sec:sol_concept}

Our target is to find a $\pi$ that maximizes $V(\pi)$ and we know $V(\pi,\psi)$ can be viewed as an estimation of $V(\pi)$ using $\psi$  so $V(\pi,\psi)$ is a natural simple objective.
However, if we consider maximizing $V(\pi,\psi)$ directly over $\Pi$ and $\Psi$, two difficulties could arise:
\begin{enumerate}[(i),topsep=1pt,parsep=1pt,itemsep=1pt,partopsep=1pt]
\item $\psi$ converges to the trivial wrong solution of $\psi(y) \to 1, \forall y \in \Ycal$, when maximizing $V(\pi,\psi)$ directly.
\item $V(\pi,\psi)$ does not necessarily correspond to the true value of $V(\pi)$. For example, if $\psi$ always decodes the feedback vector opposite to the truth, then $V(\pi,\psi)$ decreases when $V(\pi)$ increases. 
\end{enumerate}

To address these difficulties, we propose to maximize the estimated value difference from a policy $\pi_{\bad}$ known to have low expected return.  For example, a policy which chooses actions uniformly at random has an expected accuracy of $1/K$ on classification problems.  Via this policy we define the learning objective $\mathcal{L}(\pi, \psi)$ and optimal policy-decoder pair as
\begin{equation} 
\label{eq:popobj}
\argmax_{(\pi,\psi) \in \Pi \times \Psi} \Lcal(\pi,\psi) \coloneqq V(\pi,\psi) - V(\pi_{\bad},\psi).
\end{equation}
Expanding the objective for a fixed $(\pi, \psi)$ pair reveals
\begin{align}
&~ V(\pi,\psi) - V(\pi_{\bad},\psi)
\\
= &~ \E_{(x,a) \sim d_0 \times \pi}[\psi(y)] - \E_{(x,a) \sim d_0 \times \pi_{\bad}}[\psi(y)]
\\
= &~ \E_{(x,a) \sim d_0 \times \pi}[\psi(y)\1(r = 1) + \psi(y)\1(r = 0)]\\
&~ - \E_{(x,a) \sim d_0 \times \pi_{\bad}}[\psi(y)\1(r = 1) + \psi(y)\1(r = 0)]
\\
\overset{\text{(a)}}{=} &~ V(\pi)\mathbb{E}\left[\psi(y)|r=1\right] + (1 - V(\pi))\mathbb{E}\left[\psi(y)|r=0\right]
\\
&~ - V(\pi_{\bad})\mathbb{E}\left[\psi(y)|r=1\right]
\\ &~ - (1 - V(\pi_{\bad}))\mathbb{E}\left[\psi(y)|r=0\right]
\\
= &~ \left( V(\pi) - V(\pi_{\bad}) \right) (\mathbb{E}\left[\psi(y)|r=1\right] - \mathbb{E}\left[\psi(y)|r=0\right])
\\
\label{eq:anapopobj}
\doteq &~ \left( V(\pi) - V(\pi_{\bad}) \right) \Delta\psi,
\end{align}
where (a) leverages the conditional independence property in Assumption \ref{asm:condind}.

Equation~\eqref{eq:anapopobj} reveals our learning objective is a linear transformation of the unobservable quantity of interest, with intercept $V(\pi_{\bad})$ and slope $\Delta \psi \doteq (\mathbb{E}\left[\psi(y)|r=1\right] - \mathbb{E}\left[\psi(y)|r=0\right])$.  Importantly, the slope $\Delta \psi$ is independent of $\pi$, implying a single reward predictor  induces the correct ordering over all policies whose value exceeds that of $\pi_{\bad}$.  However, policies which are worse than $\pi_{\bad}$ can be ordered incorrectly, not just amongst themselves but also relative to policies that are insufficiently better than $\pi_{\bad}$. We address this next, leading to the second assumption.

\paragraph{Quality of the Reward Decoder}  Using exact expectations, any reward decoder with $\Delta \psi > 0$ induces a correct ordering over policies whose value exceeds that of $\pi_{\bad}$.  With finite sample approximations, however, a small value of $\Delta \psi$ makes this harder, resulting in increased sample complexity.  Therefore we use the slope $\Delta \psi \doteq (\mathbb{E}\left[\psi(y)|r=1\right] - \mathbb{E}\left[\psi(y)|r=0\right])$ to measure the quality of the reward decoder, and we define the optimal reward decoder $\psi^\star$ via
\begin{align}
\label{eq:psistar}
\psi^\star \coloneqq &~ \argmax_{\psi \in \Psi} \Delta\psi.
\end{align}

\paragraph{Identifiability}
Since $(V(\pi) - V(\pi_{\bad}))$ and $\Delta\psi$ are not always positive, there are potentially two extrema of equation~\eqref{eq:popobj}, one corresponding to the best policy $(V(\pi) > V(\pi_{\bad}))$ coupled with the optimal reward decoder $(\Delta \psi > 0)$, and one corresponding to the worst policy $(V(\pi) < V(\pi_{\bad}))$ coupled with the worst reward decoder $(\Delta \psi < 0)$.  To ensure the desired extrema is highest value, we make the following assumption.

\begin{assumption}[identifiability]
\label{asm:pistarstar}
There exists a constant $\eta > 0$ such that $\pi^\star$ and $\psi^\star$ satisfy
\begin{align}
\left(V(\pi^\star) - V(\pi_{\bad})\right)\Delta\psi^\star \geq V(\pi_{\bad}) + \eta.
\end{align}
\end{assumption}

\begin{remark}
\label{rmk:diff2}
Assumption \ref{asm:pistarstar} assumes that the ``incorrect'' optimization direction always achieves less $\Lcal(\pi,\psi)$ value than the ``correct'' direction, which corresponds to a random action being wrong more often than not. This can be demonstrated as follows.
Let $\pi^\dagger \coloneqq \argmin_{\pi \in \Pi} V(\pi)$ and $\psi^\dagger \coloneqq \argmin_{\psi \in \Psi} \Delta \psi$.
Since $\min_{\pi \in \Pi} V(\pi) \geq 0$ and $\min_{\psi \in \Psi} \Delta \psi \geq -1$, we have $V(\pi^\dagger,\psi^\dagger) - V(\pi_{\bad},\psi^\dagger) = ( V(\pi^\dagger) - V(\pi_{\bad}) ) \Delta\psi^\dagger \leq V(\pi_{\bad})$. Thus, Assumption \ref{asm:pistarstar} ensures $(\pi^\star, \psi^\star)$ to be the only global optima of objective Eq.\eqref{eq:popobj}, and the non-zero gap $\eta$ allows learning with finite samples to occur.
\end{remark}

The requirement from Assumption~\ref{asm:pistarstar} can be viewed as: $\pi_{\bad}$ must be ``sufficiently bad''.
For example, a policy $\pi_{\bad}$ that chooses actions uniformly at random applied on a classification task becomes increasingly bad as $K$ increases, but is never sufficiently bad when $K = 2$ because $V(\pi_{\bad}) \geq V(\pi^*)/K$ and $\Delta \psi^* \leq 1$.

\subsection{Sample Complexity}
We now provide the finite-sample results for batch-style optimization of the objective in equation \eqref{eq:popobj} with empirical data $\Dcal$. Let $\Dcal$ consists of $n$ i.i.d.~$(x,a,y)$ samples, where $(x,a)$ is generated from distribution $d(\cdot,\cdot)$. We also define $\Vhat_\Dcal(\pi,\psi) \coloneqq \frac{1}{n}\sum_{i = 1}^{n}\frac{\pi(a_i|x_i)}{d(a_i|x_i)}\psi(y_i)$ to be the estimated $V(\pi,\psi)$ using $\Dcal$.
\begin{theorem}
\label{thm:batchresult}
Let $(\pihat,\psihat) \coloneqq \argmax_{(\pi,\psi) \in \Pi \times \Psi} \Vhat_\Dcal(\pi,\psi) - \Vhat_\Dcal(\pi_{\bad},\psi)$ and
\begin{align}
\varepsilon_{\stat,n} \coloneqq &~ \sqrt{\frac{4 \max_{\pi \in \Pi}\left\|\frac{\pi(a|x)}{d(a|x)}\right\|_{2,d} \log\frac{2 |\Pi||\Psi|}{\delta} }{n}}
\\
&~ + \frac{\max_{(x,a) \in \Xcal \times \Acal} \frac{1}{d(a|x)} \log\frac{2 |\Pi||\Psi|}{\delta}}{3n}.
\end{align}
Under Assumption \ref{asm:condind} and \ref{asm:pistarstar}, if $n$ is sufficiently large such that $\varepsilon_{\stat,n} \leq \eta/2$ where $\eta$ is defined in Assumption \ref{asm:pistarstar}, then with probability $1 - \delta$:
\begin{align}
V(\pi^\star) - V(\pihat) \leq &~ \frac{2 \varepsilon_{\stat,n}}{\Delta\psi^\star}
\\
\Delta \psi^\star - \Delta \psihat \leq &~ \frac{2 \varepsilon_{\stat,n}}{V(\pi^\star) - V(\pi_{\bad})}.
\end{align}
\end{theorem}

\begin{proof}[\bf\em Proof Sketch]
We show that $\varepsilon_{\stat,n} \leq \eta/2$ is a sufficient condition of $\Delta \psihat > 0$ in Lemma \ref{lem:deltapsi} (in Appendix \ref{sec:proof}).
When we have $\Delta \psihat > 0$,  $\psihat$ will not decode the opposite reward.  Combining the facts of $(\pihat,\psihat) = \argmax_{(\pi,\psi) \in \Pi \times \Psi} \Vhat_\Dcal(\pi,\psi) - \Vhat_\Dcal(\pi_{\bad},\psi)$, $(\pi^\star,\psi^\star) = \argmax_{(\pi,\psi) \in \Pi \times \Psi} V(\pi,\psi) - V(\pi_{\bad},\psi)$, and Eq.\eqref{eq:anapopobj} yield the bound on $V(\pihat)$ and $\Delta\psihat$.
The detailed proof of Theorem \ref{thm:batchresult} is provided in Appendix \ref{sec:proof}.
\end{proof}

\begin{remark}
\label{rmk:batchidfdiff}
As we show in Theorem \ref{thm:batchresult} and its proof, maximizing $\Vhat_\Dcal(\pi,\psi) - \Vhat_\Dcal(\pi_{\bad},\psi)$ converges to the right direction only after we have sufficient data. 
This reflects the difficulty (ii) discussed at the beginning of Section \ref{sec:sol_concept} and Remark \ref{rmk:diff2}. The condition of $\varepsilon_{\stat,n} \leq \eta/2$ provides a concrete sample complexity requirement to guarantee identifiability of $(\pi^\star,\psi^\star)$, according to Assumption \ref{asm:pistarstar}.
\end{remark}


\section{Interactive Algorithms}
\label{sec:IA}

We now present an interactive algorithm for IGL.  Similar to the epoch-greedy algorithm~\citep{langford2008epoch} for contextual bandits, Algorithm \ref{alg:algoepgreedy} interleaves exploration and exploitation.  A policy that chooses actions uniformly at random is used both for exploration and as $\pi_{\bad}$ (in line with Assumption 2). Throughout this section, we use $\mu$ to denote the distribution of $(x,a) \sim d_0 \times \pi_{\bad}$, which is the data distribution of exploration data $\Dcal_i$ in Algorithm \ref{alg:algoepgreedy} at any time step $i$.

\begin{algorithm}[th]
\caption{\algoname}
\label{alg:algoepgreedy}
{\bfseries Input:} Exploration samples $\Dcal_0 = \{\}$, $t = 0$, scheduling parameters $\{n_i\}_{i = 1}^{\infty}$.
\begin{algorithmic}[1]
\For{$i = 1,2,\dotsc$}\label{algostep:step_i}
    \State Select an action uniformly at random and collect $\{(x_t, a_t, y_t)\}$.  
    \State $\Dcal_i = \Dcal_{i - 1} \cup \{(x_t, a_t, y_t)\}$.
    \State Compute $(\pi_i,\psi_i)$ by solving $$\displaystyle \argmax_{(\pi,\psi) \in \Pi\times\Psi} ~ \Eop_{(x,a,y) \sim \Dcal_i}\left[K \psi(y) \pi(a|x) \right] - \Eop_{(x,a,y) \sim \Dcal_i}\left[ \psi(y) \right].$$
    \label{step:findpolicy}
    \State Execute $\pi_i$ for $n_i$ steps (i.e., select $a_{t'} \sim \pi_i(\cdot|x_{t'})$ for $t' = t + 1, t + 2, \dotsc, t + n_i + 1$),  and set $t = t + n_i + 1$. 
\EndFor
\end{algorithmic}
\end{algorithm}

The key difference is the objective on line 4, which seeks the $\psi,\pi$ pair that most distinguishes from uniform random action values according to $\psi$ to form an unbiased estimate of the objective in  Eq.\eqref{eq:popobj}.

\paragraph{Choice of $\{n_i\}_{i = 1}^{\infty}$}
As we showed in Theorem \ref{thm:batchresult}, some amount of ``warm-up'' data is needed in order to guarantee the optimization is on the correct direction (i.e., $\Delta \psi > 0$ for the learned $\psi$). Thus, the exploitation scheduling $\{n_i\}_{i = 1}^{\infty}$ is chosen in the following way:
\begin{enumerate}[(1),topsep=1pt,parsep=1pt,itemsep=1pt,partopsep=1pt]
\item If the amount of exploration data at time step $i$ is not sufficient to guarantee $\Delta \psi_i > 0$, then explore.
\item If we have enough exploration data to ensure $\Delta \psi_i > 0$, then explore/exploit scheduling similar to the epoch-greedy algorithm~\citep{langford2008epoch} is used.
\end{enumerate}
In the analysis of \algoname, we provide the detailed definition of $\{n_i\}_{i = 1}^{\infty}$, along with a discussion about when we have sufficient exploration data to guarantee $\Delta \psi_i > 0$.

\subsection{Analysis of \algoname}

We now provide the analysis of Algorithm \ref{alg:algoepgreedy} in this section, as well as the definition and discussion of exploitation scheduling $\{n_i\}_{i = 1}^{\infty}$. Over our analysis, we use $\iota$ to denote the complexity term in a sample complexity bound for standard supervised learning, which can be naively formed as $\log\left( \frac{2T|\Psi||\Pi|}{\delta} \right)$ or using some more advanced method such as covering number or Rademacher complexity (see e.g., \citep{mohri2018foundations}).

We study the regret of Algorithm \ref{alg:algoepgreedy} which is defined as follows,
\begin{align}
\regret(T) \coloneqq T V(\pi^\star) - \E\left[\sum_{t = 1}^{T} R(x_t,a_t) \right].
\end{align}

The following theorem describes the regret guarantee of Algorithm \ref{alg:algoepgreedy}, along with the precise definition of $\{n_i\}_{i = 1}^{\infty}$.
\begin{theorem}
\label{thm:epoch_greedy}
If we choose $\{n_i\}_{i = 1}^{\infty}$ as
\begin{align}
n_i = \left\{
\begin{array}{lr}
0,     &  i \leq \nicefrac{2K^2}{\eta^2}; \\
\label{eq:epscheduling}
\lfloor \sqrt{\nicefrac{i}{K \iota}} \rfloor,     & i > \nicefrac{2K^2}{\eta^2},
\end{array}
\right.
\end{align}
then with probability at least $1 - \delta$, the regret of Algorithm \ref{alg:algoepgreedy} is bounded by 
\begin{align}
\regret(T) = \widetilde\Ocal \left(\frac{K^{\nicefrac{1}{3}}T^{\nicefrac{2}{3}} }{\Delta\psi_{\pi^\star}} + \frac{2K^2}{\eta^2} \right).
\end{align}
\end{theorem}
\begin{remark}
Although the length of initial pure exploration stage is defined using the constant $\eta$ (defined in Assumption \ref{asm:pistarstar}), there is also a data-driven way to determine it in practice, if we can upper bound the uniform policy's performance, $V(\pi_{\bad})$. That is, if $\Vhat_{\Dcal_i}(\pi_i,\psi_i) - \Vhat_{\Dcal_i}(\pi_{\bad},\psi_i) > V(\pi_{\bad}) + K \varepsilon_{\Dcal_i}$ holds at time step $i$, where $K \varepsilon_{\Dcal_i}$ denotes the statistical error at time step $i$, then we must have $\Delta \psi_i > 0$. Therefore, we can use the scheduling of $n_i = \lfloor \sqrt{\nicefrac{i}{K \iota}} \rfloor$ for all the subsequent times steps. 
The detailed theoretical basis for this is presented in Lemma \ref{lem:pistarregret_agn}.
\end{remark}

We defer the detailed proof Theorem \ref{thm:epoch_greedy} to Appendix \ref{sec:proof}. The following proof sketch describes the main technical components for proving Theorem \ref{thm:epoch_greedy}.
\begin{proof}[\bf\em Proof Sketch]
We use $\Vhat_\Dcal(\pi,\psi)$ and $\Vhat_\Dcal(\pi_{\bad},\psi)$ to denote the empirical estimations of $V(\pi,\psi)$ and $V(\pi_{\bad},\psi)$ in Algorithm \ref{alg:algoepgreedy} respectively, where
\begin{align}
\Vhat_\Dcal(\pi,\psi) \coloneqq &~ \Eop_{(x,a,y) \sim \Dcal}\left[ K \psi(y) \pi(a|s) \right]
\\
\Vhat_\Dcal(\pi_{\bad},\psi) \coloneqq &~ \Eop_{(x,a,y) \sim \Dcal}\left[ \psi(y) \right].
\end{align}
Then step \ref{step:findpolicy} in Algorithm \ref{alg:algoepgreedy} can we rewritten as
\begin{align}
\label{eq:optobj}
\argmax_{\pi \in \Pi} \max_{\psi \in \Psi} \qquad \Vhat_\Dcal(\pi,\psi) - \Vhat_\Dcal(\pi_{\bad},\psi_\pi),
\end{align}
where $\Dcal$ denotes $\Dcal_i$ for specific time step $i$.

For simplicity, we also define $\psihat_{\Dcal}$ and $\pihat_\Dcal$ to be the learned policy and reward decoder given exploration data $\Dcal$ (for specific time step $i$, also just set $\Dcal = \Dcal_i$),
\begin{align}
\label{eq:defpihatpsihat}
(\pihat_\Dcal, \psihat_{\Dcal}) \coloneqq &~ \argmax_{(\pi,\psi) \in \pi \times \Psi} \Vhat_\Dcal(\pi,\psi) - \Vhat_\Dcal(\pi_{\bad},\psi_\pi).
\end{align}

Over this section, we define $\varepsilon_\Dcal$ as $\varepsilon_{\Dcal} \coloneqq \sqrt{\nicefrac{\iota}{2|\Dcal|}}$, and we have $|(\Vhat_\Dcal(\pi,\psi) - \Vhat_\Dcal(\pi_{\bad},\psi)) - (V(\pi,\psi) - V(\pi_{\bad},\psi))| \leq K \varepsilon_{\Dcal}$ for any $(\pi,\psi) \in \Pi \times \Psi$ by simply applying the standard concentration inequality.

The difficulty of identifiability discussed for the batch mode (see Remark \ref{rmk:batchidfdiff}), implies that Algorithm \ref{alg:algoepgreedy} should not start exploitation until it gathers enough exploration data so that better than random performance can be guaranteed after learning from the exploration data. We formalize this fact with the following lemmas.

\begin{lemma}
\label{lem:decodebadpi}
Let $\Dcal$ be the exploration data at arbitrary time step. If $\pi$ satisfies $V(\pi) \leq V(\pi_{\bad})$, then, with probability at least $1 - \delta$, we have for any $\psi \in \Psi$,
\begin{align}
\label{eq:badpivtilde}
\Vhat_\Dcal(\pi,\psi) - \Vhat_\Dcal(\pi_{\bad},\psi) \leq V(\pi_{\bad}) + K \varepsilon_{\Dcal}.
\end{align}
\end{lemma}
In Lemma \ref{lem:decodebadpi}, we upper bound the value of the objective function if we enter the ``opposite'' optimization direction --- $\psi \to \argmin_{\psi \in \Psi} \Delta \psi < 0$. The proof of Lemma \ref{lem:decodebadpi} follows a similar argument as Remark \ref{rmk:diff2} and is deferred to Appendix \ref{sec:proof}. By using that result, the next lemma shows that if the amount of exploration data is large enough for $\Vhat_\Dcal(\pihat_{\Dcal},\psihat_{\Dcal}) - \Vhat_\Dcal(\pi_{\bad},\psihat_{\Dcal}) > V(\pi_{\bad}) + K \varepsilon_{\Dcal}$, then we must have $\Delta \psihat_{\Dcal} > 0$, which yields the bound on $V(\pihat_{\Dcal})$.
\begin{lemma}
\label{lem:pistarregret_agn}
Let $\Dcal$ be the exploration data at some time step with $(\pihat_\Dcal, \psihat_\Dcal)$ defined as in Eq.\eqref{eq:defpihatpsihat}. Then, if
\begin{align}
\label{eq:pitildegvb_a}
\Vhat_\Dcal(\pihat_{\Dcal},\psihat_{\Dcal}) - \Vhat_\Dcal(\pi_{\bad},\psihat_{\Dcal}) > V(\pi_{\bad}) + K \varepsilon_{\Dcal},
\end{align}
then we have,
\begin{align}
V(\pihat_{\Dcal}) \geq V(\pi^\star) - \frac{2 K \varepsilon_\Dcal}{\Delta \psi^\star}.
\end{align}
\end{lemma}



The next lemma provides a bound on the amount of initial exploration data that we need to ensure Eq.\eqref{eq:pitildegvb_a}.
\begin{lemma}
\label{lem:enterlowreg}
Let $T_0 = {2K^2 \iota}/{\eta^2}$ where $\eta$ is defined in Assumption \ref{asm:pistarstar}, then with probability at least $1 - \delta$, we have
\begin{align}
\Vhat_\Dcal(\pihat_{\Dcal},\psihat_{\Dcal}) - \Vhat_\Dcal(\pi_{\bad},\psihat_{\Dcal}) > V(\pi_{\bad}) + K \varepsilon_{\Dcal},
\end{align}
where $\Dcal = \Dcal_{t > T_0}$.
\end{lemma}
Combining these lemmas together with a similar argument as in \citep{langford2008epoch}, we establish a proof for Theorem \ref{thm:epoch_greedy}. The detailed proof of all the lemmas above and Theorem \ref{thm:epoch_greedy} can be found in Appendix \ref{sec:proof}.
\end{proof}

\section{Experiments}
\label{sec:experiments}

In this section, we provide empirical evaluations in simulated environments. We experiment with both batch and online IGL.



The task is as depicted in Figure~\ref{fig:diff_setting}. We evaluated our approach by comparing
\begin{enumerate}[(1),topsep=1pt,parsep=1pt,itemsep=1pt,partopsep=1pt]
    \item {\sf SUP} --- Supervised classification, as in Figure~\ref{fig:supln};
    \item {\sf CB} --- Contextual bandits with exact reward, as in Figure~\ref{fig:cbs};
    \item {\sf IGL} --- Our proposed approach using Eq.\eqref{eq:popobj} for batch mode learning and Algorithm \ref{alg:algoepgreedy} for the online setting, with an image feedback vector as in Figure~\ref{fig:bci}.
\end{enumerate}
Note that supervised learning ({\sf SUP}) should be better than contextual bandit learning ({\sf CB}), which in turn should  do better than Interaction-Grounded Learning ({\sf IGL}) since each step in that sequence makes the problem more difficult.

Supervised classification uses logistic regression with a linear representation and cross-entropy loss. The other methods use the same representation with softmax policies. During testing time, each algorithm takes the argmax of the policy. We provide the details on setting up the experiments in Appendix \ref{sec:expdetail}, where we also discuss the practical difficulty of jointly optimizing $\pi$ and $\psi$ created by the multiple extrema of equation~\eqref{eq:popobj} and propose some mitigation strategies.
%



\subsection*{Experimental Results}
We evaluated our approach on the MNIST environment based on the infinite MNIST simulator \citep{loosli2007training}.
At each time step, the context $x_t$ is generated uniformly from the infinite MNIST simulator. After that, the learner selects action $a_t \in \{0, ... , 9\}$ as the predicted label of $x_t$. The binary reward $r$ is the correctness of the prediction label $a_t$. The feedback vector $y_t$ is also generated from the infinite MNIST simulator, either an image of a one digit or an image of a zero digit depending upon $r$.

Over our experiments in the batch mode, we use the uniform policy $\pi_{\bad}$ to gather data, and the number of examples is $60000$. Our results are averaged over 16 random trials.

\begin{table}[ht]
\centering
\begin{tabular}{cc}
\hline
\textbf{Setting}  & \textbf{Policy Accuracy (\%)} \\ \hline
\textbf{{\sf SUP}}      & 90.62 $\pm$ 1.02            \\
\textbf{{\sf CB}}          & 85.58 $\pm$ 4.50           \\ 
\textbf{{\sf IGL}}       & 82.21 $\pm$ 4.33            \\ \hline
\end{tabular}
\caption{Results in Batch Mode of MNIST Environment.}
\label{tb:mnist}
\end{table}

Table \ref{tb:mnist} is the result in the batch mode of the MNIST environment. All the experiments are repeated 16 times with the average and standard deviation reported.
For our {\sf IGL} algorithm, we optimize the objective in line 4 of \algoname on the batch data.
{\sf IGL} achieves comparable accuracy as {\sf CB} despite the handicap of only observing feedback vectors. 
%



\begin{figure}[thb]
    \centering
        \includegraphics[width=0.9\linewidth]{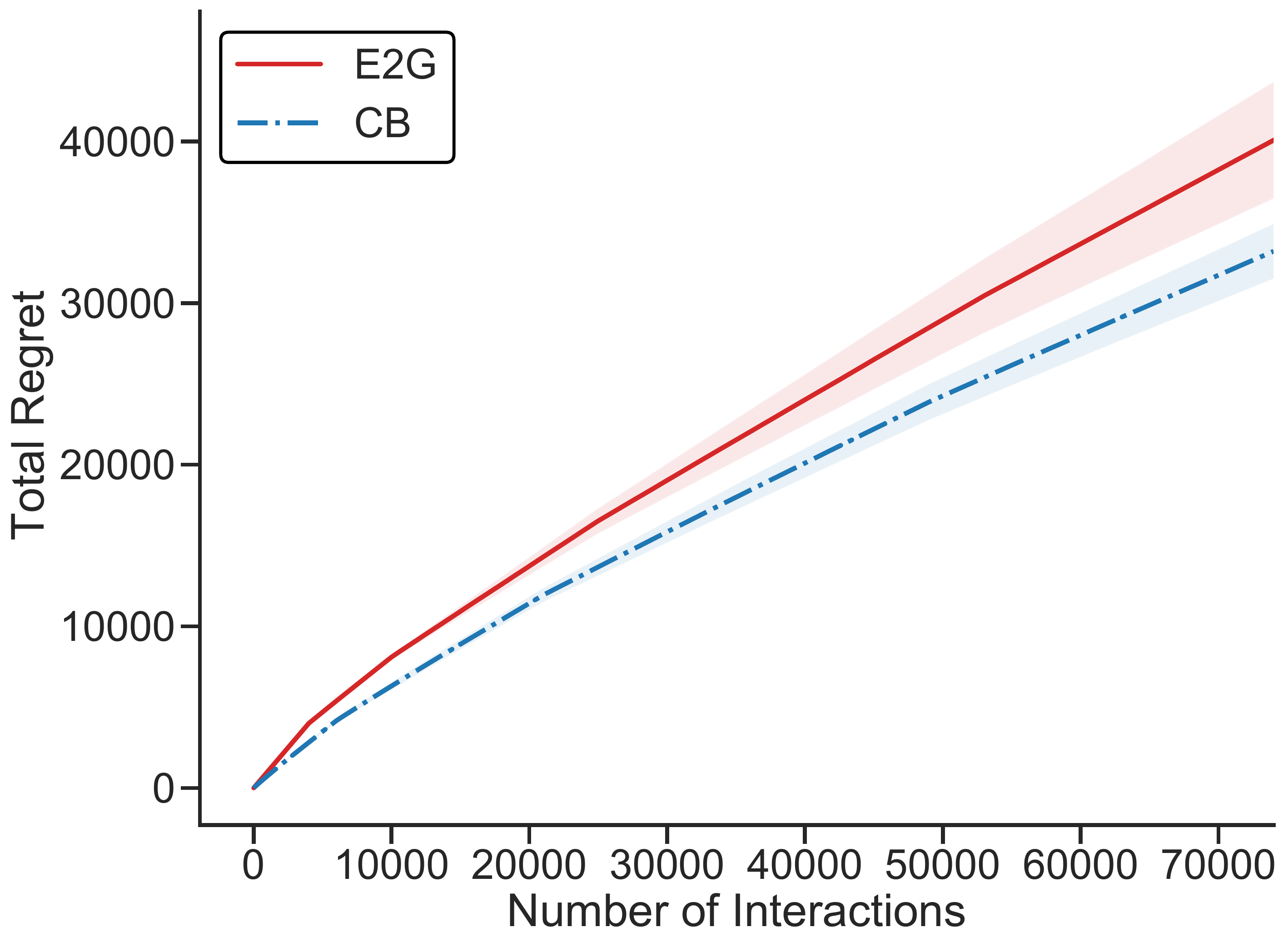}
    \caption{Comparison of {\algoname} and competitors on the MNIST environment.}
    \label{fig:ep_greedy_mnist}
\end{figure}

We also compare {\algoname} with a {\sf CB} algorithm in the online setting with the results shown in Figure \ref{fig:ep_greedy_mnist} (averaged over 16 runs). In this case, {\algoname} starts with 4000 exploration events based on the suggestion of Theorem \ref{thm:epoch_greedy}. This result demonstrates the effectiveness of the online use of {\algoname}.


\section{Failure of Unsupervised Learning in IGL}
\label{sec:failunsup}

Unsupervised learning provides another approach to Interaction-Grounded Learning. If unsupervised learning distinguishes the feedback vectors generated from different rewards, then learning the optimal policy for IGL is still possible.  Indeed, the task from section~\ref{sec:experiments}  can be solved by clustering the feedback vectors.  However, in this section, we show the information-theoretical hardness of using unsupervised learning in IGL in general.

To formalize our argument, suppose the agent picks an unsupervised-learning oracle and a contextual-bandits oracle. The reward is decoded from feedback vectors using the unsupervised-learning oracle, and the contextual-bandits oracle learns a policy using the decoded reward. 
Our result focuses on the ambiguity of unsupervised-learning-based approaches. To make it cleaner, we consider the scenario of (1) data is infinite, (2) the contextual-bandits oracle could always output the exact optimal policy. Note that, this result could be extended to the general case by capturing the statistical error and approximation error properly.
The following theorem formally present our lower bound result.

\begin{theorem}
\label{thm:lowerbound}
There exists a set of IGL tasks and a bad behavior policy $\pi_\bad$, such that: (i) each IGL task and $\pi_\bad$ satisfy Assumption \ref{asm:condind} and Assumption \ref{asm:pistarstar}, the data for each IGL task is infinite, and data is gathered by running $\pi_\bad$ (ii) the contextual bandit oracle outputs the exact optimal policy corresponding to the input reward, then for any unsupervised learning oracle called by the unsupervised-learning-based approach, there exists at least one IGL task in that class, such that the performance loss of output policy, $\pihat$, is $V(\pi^\star) - V(\pihat) = \Omega(1)$.
\end{theorem}
\begin{proof}[\bf\em Proof of Theorem \ref{thm:lowerbound}]
We construct the following 10 environments based on the MNIST environment introduced in Section \ref{sec:experiments}. Each of theses 10 environments have the same context-action-reward setting as described in Section \ref{sec:experiments} while differing in the feedback vector generation process. The environment $i$ ($i = 0,1,\dotsc,9$) generates the feedback vector in the following way:
\begin{align}
y|_{r = 0} = &~ \text{a random image with label  \{``0''$,\dotsc,$``9''\}$\setminus$``$i$''},
\\
&~ \text{where the labels are also distributed uniformly},
\\
y|_{r = 1} = &~ \text{a random image with label ``$i$''}.
\end{align}

If the data gathering policy, $\pi_\bad$, is the uniform policy over 10 actions, then the distribution of $y$ is \emph{identical} for all 10 environments above. It implies that any unsupervised-learning-based approach is not able to distinguish among these 10 environments no matter which unsupervised-learning oracle it calls, and therefore it would decode the same reward for all these 10 environments.

We use $\psi(\cdot)$ to denote the learned reward decoder, and let $r_i \coloneqq \E[\psi(y)|\text{image with label ``$i$''}], \forall i \in \{0,1,\dotsc,9\}$, where $r_i\in [0,1]$ are the expected decoded reward for images with label ``$i$''.

Without loss of generality, let $\argmin_{i \in \{0,1,\dotsc,9\}} r_i = r_0$. Since the contextual bandits oracle output the exact optimal policy with the corresponding decoded reward, we can obtain that the output policy for the environment $0$, $\pihat_0$, has $V(\pihat_0) = 0$. That is because $\E[\psi(y)|r = 1] = r_0 < \E[\psi(y)|r = 0] = \frac{1}{9}(r_1 + r_2 + \dotsc + r_9)$. This completes the proof.
\end{proof}

\section{Related Work}

The problem of partial monitoring \citep{mertens1990repeated,piccolboni2001discrete,mannor2003line,cesa2006regret,bartok2014partial,lattimore2019information,lattimore2020bandit} also provides a framework for the decision-making problems with imperfect feedback. Most of the papers on the partial monitoring problem consider the case where the feedback is a known function of the actual cost, resulting in algorithms that compose the known function with more standard online learning strategies.  Notable exceptions are \citet{hanawal2017unsupervised,verma2019online,verma2020online} which study an unsupervised sequential bandit setting where feedback information does not directly identify the underlying arm reward.  They identify a condition where pairwise disagreement among the binary components of the feedback vector is sufficient to order the arms correctly.  The IGL setting includes these prior works as special cases.

Another related setting is latent (contextual) bandits \citep{maillard2014latent,zhou2016latent,hong2020latent}.
In the problem of latent (contextual) bandits, the reward is not only observed, but is drawn from a known distribution conditioned on a latent state. The primary goal of this problem is to identify that latent state, such that acting optimally is straightforward, and learning proceeds more rapidly than naive utilization of the observed reward.   IGL distinguishes from this setting as the reward is still observed in the latent (contextual) bandits, whereas IGL must infer the latent reward from interaction.

Inverse reinforcement learning \citep{ng2000algorithms} learns a reward which explains the behavior of expert demonstrations for the purpose of learning a control policy.  In the IGL setting, there are no expert demonstrations; instead there is joint learning of a reward decoder and policy from interaction data.

Other authors have investigated alternatives to rewards for agent behavior declaration such as via convex constraints~\cite{convex} or to satisfy multiple objectives~\cite{multi-objective}.  The specification and nature of the feedback in IGL is less structured than in these settings.

Language games attempt to model the emergence of grounded communication between multiple agents cooperating to succeed in a task where each agent has partial information~\cite{nowak1999evolutionary,bouchacourt2018agents}.  In language games, reward is observed and grounding is required to communicate context information; whereas in IGL the reward is unobserved and the context is fully observed.  


\section{Discussion}\label{sec:discussion}

We have proposed a novel setting, Interaction-Grounded Learning, in which a learner learns to interact with the environment in the absence of any explicit reward signals. The learner observes a context vector from the environment, takes an action, and observes a feedback vector. Without having a grounding for the feedback, the agent makes the assumption that there is latent reward signal in the feedback. With a conditional independence assumption on this latent reward, the agent uses an algorithm to discover this latent reward in order to ground its policies. We have proposed the E2G algorithm and proven that when the assumptions are met, it can solve IGL.

This work leverages the assumption of \emph{conditional independence} of the feedback vector from the context and action given the reward.  Of the required  assumptions, this assumption may appear the most restrictive.  This assumption is reasonable for some problems (e.g., a smart speaker reacting to idiosyncratic user vocabulary indicating approval or disapproval) and unreasonable for others (e.g., in a BCI application, EEG signals exhibit autocorrelation).  However, in many signal detection application areas including EEG-based neurofeedback, which is a motivating application of this setting, it is common practice to postprocess the data to approximately satisfy this assumption by regressing out conditions and analyzing residuals. 

Relaxing the assumption of conditional independence is a direction for future work.  For problems where the feedback vector is influenced by the context and action, it may be possible to synthesize a conditionally independent signal, e.g., via variational approximations to mutual information~\cite{belghazi2018mine} or via regression residuals~\cite{shah2020hardness}.  Indeed semi-parametric regression approaches are used pervasively in functional neuroimaging studies. Neural signals corresponding to successive conditions are commonly orthogonalized using a General Linear Model \cite{Momennejad2013} or Finite Impulse Response \cite{Momennejad2012}. Regression and residual approaches thus orthogonalize the signal prior to further analysis with other machine learning methods. In future work we hope to develop an approach to sanitize the signal to ensure conditional independence is satisfied. 

Interaction-grounded learning can be applied to many interesting domains. Among them are applications to BCI, prosthetics, and neurofeedback. BCI and Neurofeedback have been applied to improve memory \cite{memoryeeg2015}, train attention \cite{closedloop2015, deBettencourt2015}, and facilitate learning and memory during sleep \cite{antony_sleep_eeg2018}. For example, in functional neuroimaging studies of closed-loop neurofeedback, the content on screen reacts to the neural signals of the human participant in order to help the participant gain control of a given neural state, e.g. in training attention \cite{closedloop2015, neurofeedback2020}. To do so, a classifier is often extensively trained on numerous labeled samples of neural data (e.g. real-time fMRI signals of attentional state) before it is applied to reading and reacting to brain data \cite{deBettencourt2015}. With IGL, successful feedback (e.g. neurofeedback) may be achievable without supervised training. 

Broadly, in all previous BCI work the feedback was grounded in extensive training. This is not ideal in many scenarios, such as the case of locked-in patients who may be conscious but lack the ability to ground their neural signals in other responses. IGL may be specifically helpful for such cases. Moreover, in neurofeedback applications \cite{neurofeedback2020} typically a human participant learns to calibrate their neural response to the feedback of an algorithm. Conversely, Interaction-Grounded Learning is a setting where the learner algorithm is required to calibrate itself to interpret un-grounded feedback from the environment, e.g. human EEG signals. IGL opens up possibilities for these and broader future directions, both in terms of research and application. Further examples follow.

Applications of IGL in more standard human-computer interface problems are potentially powerful.  An example problem here is interpreting human gestures---people learn and use many gestures in a personalized way when working with each other.  Could a robot naturally learn to interpret the gestures of a human using IGL techniques?  Could the operating system of a laptop computer use IGL techniques to improve the interpretation of mouse, touchscreen, and/or viewed gestures?  Realizing the benefits of IGL here requires a paradigm shift away from designed interfaces towards learned interfaces.  The most extreme example of a designed interface is perhaps a keyboard. The keyboard has been very successful, yet the shift to small form factor compute devices has made keyboards significantly more awkward necessitating the development of new kinds of interfacing which inherently suffer from more ambiguity.  Ambiguities in interpreting touch-screen gestures, handwriting, speech, or gestures and body language are areas where interaction-grounded learning may be valuable.

Finally, it is worth noting that in this work we assume a stationary environment.  A two-agent IGL scenario is equivalent to one agent being oblivious, while the other agent might try to adjust the feedback vector to help grounding succeed.  This beneficial learning variant of IGL could overlap with the language game literature, if the reward is considered privileged information for one of the agents.  The language game variant of IGL is a fascinating topic for future theoretical and empirical investigations.


\bibliography{ref}
\bibliographystyle{icml2021}


\clearpage

\appendix
\onecolumn

\begin{center}
{\Large Appendix}
\end{center}


\section{Detailed Proofs}
\label{sec:proof}

\subsection{Proofs for Section \ref{sec:LfCI}}

Before providing the proof of Theorem \ref{thm:batchresult}, we first present the following lemma to show that $\varepsilon_{\stat,n} < \eta/2$ is a sufficient condition of $\Delta \psihat > 0$. 
\begin{lemma}
\label{lem:deltapsi}
Let $(\pihat,\psihat)$ and $\varepsilon_{\stat,n}$ are defined same as Theorem \ref{thm:batchresult}, and $n = |\Dcal|$. If $n$ is large enough such that $\varepsilon_{\stat,n} < \eta/2$, we have $\Delta \psihat > 0$.
\end{lemma}
\begin{proof}[\bf\em Proof of Lemma \ref{lem:deltapsi}]
By the optimality of $(\pihat,\psihat)$, we have
\begin{align}
\Vhat_\Dcal(\pihat,\psihat) - \Vhat_\Dcal(\pi_u,\psihat) \geq &~ \Vhat_\Dcal(\pi^\star,\psi^\star) - \Vhat_\Dcal(\pi_u,\psi^\star)
\\
\Longrightarrow V(\pihat,\psihat) - V(\pi_u,\psihat) \geq &~ V(\pi^\star,\psi^\star) - V(\pi_u,\psi^\star)  - 2 \varepsilon_{\stat,n}.
\end{align}
Plugging in the definition of $V(\pihat,\psihat)$ and combining with Eq.\eqref{eq:anapopobj}, we obtain
\begin{align}
\left( V(\pihat) - V(\pi_u) \right) \Delta\psihat \geq &~ \left( V(\pi^\star) - V(\pi_u) \right) \Delta\psi^\star - 2 \varepsilon_{\stat,n}
\\
\overset{\text{(a)}}{\geq} &~ V(\pi_u) + \eta - 2 \varepsilon_{\stat,n}
\\
> &~ V(\pi_u) \tag{$\varepsilon_{\stat,n} < \eta/2$}
\\
\geq &~ (0 - V(\pi_u)) \cdot (-1)
\\
\label{eq:deltaphidagger}
\overset{\text{(b)}}{\geq} &~ \left( V(\pi^\dagger) - V(\pi_u) \right) \Delta\psi^\dagger,
\end{align}
where (a) follows from Assumption \ref{asm:pistarstar}, i.e.,$\left(V(\pi^\star) - V(\pi_u)\right)\Delta\psi^\star \geq V(\pi_u) + \eta$, and (b) follows from that fact of $V(\pi^\dagger) \coloneqq \min_{\pi \in \Pi} V(\pi) \geq 0$ and $\Delta\psi^\dagger \coloneqq \min_{\psi \in \Psi} \Delta\psi \geq -1$. 

If $\Delta \psihat < 0$, we must have $|\Delta \psihat| \leq |\Delta\psi^\dagger|$ by the definition of $\psi^\dagger$, and then $\left( V(\pi) - V(\pi_u) \right) \Delta\psihat$ will be no greater than the RHS of Eq.\eqref{eq:deltaphidagger} for any $\pi \in \Pi$. Therefore, $\Delta \psihat < 0$ contradicts the results of Eq.\eqref{eq:deltaphidagger}, and so we must have $\Delta\psihat > 0$.
\end{proof}


\begin{proof}[\bf\em Proof of Theorem \ref{thm:batchresult}]
For any $\pi$ and $\psi$, we have the following results with probability $1 - \delta$,
\begin{align}
&~ \left| \left(\Vhat_\Dcal(\pi,\psi) - \Vhat_\Dcal(\pi_u,\psi)\right) - \left(V(\pi,\psi) - V(\pi_u,\psi)\right) \right|
\\
\leq &~ \left| \frac{1}{n}\sum_{i = 1}^{n}\frac{\pi(a_i|x_i) - \frac{1}{K}}{d(a_i|x_i)}\psi(y_i) - \Eop_{(x,a) \sim d(\cdot,\cdot)}\left[ \frac{1}{n}\sum_{i = 1}^{n}\frac{\pi(a|x) - \frac{1}{K}}{d(a|x)}\psi(y) \right] \right|
\\
\leq &~ \sqrt{\frac{2 \Var_{d} \left[ \frac{\pi(a|x) - \frac{1}{K}}{d(a|x)}\psi(y) \right] \log\frac{2 |\Pi||\Psi|}{\delta} }{n}} + \frac{\max_{(x,a) \in \Xcal \times \Acal}\left| \frac{\pi(a|x) - \frac{1}{K}}{d(a|x)}\psi(y) \right| \log\frac{2 |\Pi||\Psi|}{\delta}}{3n} \tag{Bernstein's inequality}
\\
\leq &~ \sqrt{\frac{4 \max_{\pi \in \Pi}\left\|\frac{\pi(a|x)}{d(a|x)}\right\|_{2,d} \log\frac{2 |\Pi||\Psi|}{\delta} }{n}} + \frac{\max_{(x,a) \in \Xcal \times \Acal} \frac{1}{d(a|x)} \log\frac{2 |\Pi||\Psi|}{\delta}}{3n} \eqqcolon \varepsilon_{\stat,n}.  \tag{$\psi(\cdot) \in [0,1],~ \forall \psi \in \Psi$ }
\end{align}

By the optimality of $(\pihat,\psihat)$, we have
\begin{align}
\Vhat_\Dcal(\pihat,\psihat) - \Vhat_\Dcal(\pi_u,\psihat) \geq &~ \Vhat_\Dcal(\pi^\star,\psi^\star) - \Vhat_\Dcal(\pi_u,\psi^\star)
\\
\geq &~ V(\pi^\star,\psi^\star) - V(\pi_u,\psi^\star) - \varepsilon_{\stat,n}
\\
\label{eq:batchpihatpistar}
= &~  \left( V(\pi^\star) - V(\pi_u) \right) \Delta\psi^\star - \varepsilon_{\stat,n}.
\end{align}

On the other hand,
\begin{align}
\Vhat_\Dcal(\pihat,\psihat) - \Vhat_\Dcal(\pi_u,\psihat) \leq &~ V(\pihat,\psihat) - V(\pi_u,\psihat) + \varepsilon_{\stat,n}
\\
\label{eq:batchpihat}
= &~ \left( V(\psihat) - V(\pi_u) \right) \Delta\psihat + \varepsilon_{\stat,n}
\end{align}

Therefore, combining Eq.\eqref{eq:batchpihatpistar} and Eq.\eqref{eq:batchpihat}, we obtain for $\pihat$
\begin{align}
\left( V(\pihat) - V(\pi_u) \right) \Delta\psihat + \varepsilon_{\stat,n} \geq &~ \left( V(\pi^\star) - V(\pi_u) \right) \Delta\psi^\star
\\
\label{eq:tm1temp1}
\Longrightarrow \left( V(\pihat) - V(\pi_u) \right) \Delta\psihat \geq &~ \left( V(\pi^\star) - V(\pi_u) \right) \Delta\psi^\star - 2 \varepsilon_{\stat,n}
\\
\overset{\text{(a)}}{\Longrightarrow} \left( V(\pihat) - V(\pi_u) \right) \Delta\psi^\star \geq &~ \left( V(\pi^\star) - V(\pi_u) \right) \Delta\psi^\star - 2 \varepsilon_{\stat,n}
\\
\Longrightarrow V(\pihat) \geq &~ V(\pi^\star) - \frac{2 \varepsilon_{\stat,n}}{\Delta\psi^\star},
\end{align}
where (a) requires $\Delta\psihat > 0$ which is following Lemma \ref{lem:deltapsi}.

For $\psihat$, we have
\begin{align}
\left( V(\pihat) - V(\pi_u) \right) \Delta\psihat + \varepsilon_{\stat,n} \geq &~ \left( V(\pi^\star) - V(\pi_u) \right) \Delta\psi^\star
\\
\Longrightarrow \left( V(\pihat) - V(\pi_u) \right) \Delta\psihat \geq &~ \left( V(\pi^\star) - V(\pi_u) \right) \Delta\psi^\star - 2 \varepsilon_{\stat,n}
\\
\overset{\text{(b)}}{\Longrightarrow} \left( V(\pi^\star) - V(\pi_u) \right) \Delta\psihat \geq &~ \left( V(\pi^\star) - V(\pi_u) \right) \Delta\psi^\star - 2 \varepsilon_{\stat,n}
\\
\Longrightarrow \Delta\psihat \geq &~ \Delta\psi^\star - \frac{2 \varepsilon_{\stat,n}}{V(\pi^\star) - V(\pi_u)},
\end{align}
where (b) requires $V(\pihat) - V(\pi_u) > 0$ which can be obtained directly following $\Delta\psihat > 0$. This completes the proof.
\end{proof}


\subsection{Proofs for Section \ref{sec:IA}}

\begin{proof}[\bf\em Proof of Lemma \ref{lem:decodebadpi}]
By the Hoeffding's equality, with probability at least $1 - \delta$, we have the following inequalities for any $\pi \in \Pi$ with $V(\pi) \leq V(\pi_u)$ and $\psi \in \Psi$
\begin{align}
\Vhat_\Dcal(\pi,\psi) - \Vhat_\Dcal(\pi_u,\psi) = &~ \Eop_{(x,a,y) \sim \Dcal}\left[ K \psi(y) \pi(a|x) - \psi(y) \right]
\\
\leq &~ \Eop_{(x,a,y) \sim \mu}\left[ K \psi(y) \pi(a|x) - \psi(y) \right] + K \varepsilon_\Dcal
\\
\leq &~ \left( V(\pi) - V(\pi_u) \right) (\psi_1 - \psi_0) + K \varepsilon_\Dcal
\\
\leq &~ V(\pi_u) + K \varepsilon_\Dcal. \tag*{\qedhere}
\end{align}
\end{proof}


\begin{proof}[\bf\em Proof of Lemma \ref{lem:pistarregret_agn}]
%
Let $\psihat_{\pihat_\Dcal,\Dcal,1} \coloneqq \Ebb\left[ \psihat_{\Dcal}(y)\middle|r = 1 \right]$ and $\psihat_{\pihat_\Dcal,\Dcal,0} \coloneqq \Ebb\left[ \psihat_{\Dcal}(y)\middle|r = 0 \right]$, then, on one hand we have.
\begin{align}
\Vhat_\Dcal(\pihat_\Dcal,\psihat_{\Dcal}) - \Vhat_\Dcal(\pi_u,\psihat_{\Dcal}) \geq &~ \Vhat_\Dcal({\pi^\star},\psi_{\pi^\star}) - \Vhat_\Dcal(\pi_u,\psi_{\pi^\star})
\\
\geq &~ V(\pi^\star,\psi^\star) - V(\pi_u,\psi^\star) - K \varepsilon_\Dcal
\\
= &~ (\psi^\star_1 - \psi^\star_0) \left(V(\pi^\star) - V(\pi_u)\right) - K \varepsilon_\Dcal.
\\
\label{eq:bdpitilde1_agn}
= &~ \Delta\psi^\star \left(V(\pi^\star) - V(\pi_u)\right) - K \varepsilon_\Dcal.
\end{align}
On the other hand,
\begin{align}
\Vhat_\Dcal(\pihat_\Dcal,\psihat_{\Dcal}) - \Vhat_\Dcal(\pi_u,\psihat_{\Dcal}) \leq &~ V(\pihat_\Dcal,\psihat_{\Dcal}) - V(\pi_u,\psihat_{\Dcal}) + K \varepsilon_\Dcal
\\
= &~ (\psihat_{\pihat_\Dcal,\Dcal,1} - \psihat_{\pihat_\Dcal,\Dcal,0}) \left( V(\pihat_\Dcal) - V(\pi_u) \right) + K \varepsilon_\Dcal.
\\
\label{eq:bdpitilde2_agn}
= &~ \Delta\psihat_{\Dcal} \left( V(\pihat_\Dcal) - V(\pi_u) \right) + K \varepsilon_\Dcal.
\end{align}

By assuming Eq.\eqref{eq:pitildegvb_a}, we know
\begin{align}
\Vhat_\Dcal(\pihat_\Dcal,\psihat_{\Dcal}) - \Vhat_\Dcal(\pi_u,\psihat_{\Dcal}) > &~ V(\pi_u) + K \varepsilon_{\Dcal}
\\
\Longrightarrow V(\pihat_\Dcal,\psihat_{\Dcal}) - V(\pi_u,\psihat_{\Dcal}) + \varepsilon_{\Dcal} > &~ V(\pi_u) + K \varepsilon_{\Dcal}
\\
\Longrightarrow V(\pihat_\Dcal,\psihat_{\Dcal}) - V(\pi_u,\psihat_{\Dcal}) > &~ V(\pi_u)
\\
\label{eq:deltaphitilde_last}
\Longrightarrow \Delta\psihat_{\Dcal} \left( V(\pihat_\Dcal) - V(\pi_u) \right) > &~ V(\pi_u)
\\
\label{eq:deltaphitilde}
\Longrightarrow \Delta\psihat_{\Dcal} > &~ 0,
\end{align}
where the last inequality follows from the fact that $V(\pihat_\Dcal) - V(\pi_u) \geq - V(\pi_u)$ and $\Delta\psihat_{\Dcal} \geq -1$, and therefore, the LHS of Eq.\eqref{eq:deltaphitilde_last} is at most $1/K$ if $\Delta\psihat_{\Dcal} < 0$.

Combining Eq.\eqref{eq:bdpitilde1_agn}, Eq.\eqref{eq:bdpitilde2_agn}, and Eq.\eqref{eq:deltaphitilde}, we obtain
\begin{align}
\Delta\psihat_{\Dcal} \left( V(\pihat_\Dcal) - V(\pi_u) \right) + K \varepsilon_\Dcal \geq &~ \Delta\psi^\star \left(V(\pi^\star) - V(\pi_u)\right) - K \varepsilon_\Dcal
\\
\Longrightarrow V(\pihat_\Dcal) \geq &~ \frac{\Delta\psi^\star \left(V(\pi^\star) - V(\pi_u)\right) - 2 K \varepsilon_\Dcal}{\Delta\psihat_{\Dcal}}
\\
\geq &~ \frac{\Delta\psi^\star \left(V(\pi^\star) - V(\pi_u)\right) - 2 K \varepsilon_\Dcal}{\Delta\psi^\star}
\\
= &~ V(\pi^\star) - \frac{2 K \varepsilon_\Dcal}{\Delta\psi^\star}.
\end{align}
This completes the proof.
\end{proof}

\begin{proof}[\bf\em Proof of Lemma \ref{lem:enterlowreg}]
Let $\Dcal$ is the $\Dcal_i$ with $t > T_0 = {2K^2 \iota}/{\eta^2}$, then we have $2 K \varepsilon_{\Dcal} \leq \eta$ with probability at least $1 - \delta$, by following the definition of $\varepsilon_\Dcal$ ($\varepsilon_{\Dcal} \coloneqq \sqrt{\nicefrac{\iota}{2|\Dcal|}}$). Therefore, by Assumption \ref{asm:pistarstar}, we obtain
\begin{align}
\left(V(\pi^\star) - V(\pi_u)\right)\Delta\psi_{\pi^\star} > &~ V(\pi_u) + 2 K \varepsilon_{\Dcal}
\\
\Longrightarrow V(\pi^\star,\psi_{\pi^\star}) - V(\pi_u,\psi_{\pi^\star}) > &~ V(\pi_u) + 2K \varepsilon_{\Dcal}
\\
\Longrightarrow \Vhat_\Dcal(\pi^\star,\psi_{\pi^\star}) - \Vhat_\Dcal(\pi_u,\psi_{\pi^\star}) > &~ V(\pi_u) + K \varepsilon_{\Dcal}
\\
\Vhat_\Dcal(\pihat_\Dcal,\psihat_{\Dcal}) - \Vhat_\Dcal(\pi_u,\psihat_{\Dcal}) > &~ V(\pi_u) + K \varepsilon_{\Dcal}.
\end{align}
This completes the proof.
\end{proof}


\begin{proof}[\bf\em Proof of Theorem \ref{thm:epoch_greedy}]
By Lemma \ref{lem:pistarregret_agn}, we can bound the one-step regret of executing any $\pihat_{\Dcal_i}$ by $\frac{2K \varepsilon_\Dcal}{\Delta\psi_{\pi^\star}}$, as long as $t > T_0 = {2K^2 \iota}/{\eta^2}$.
Therefore, the scheduling of $\{n_t\}_{t = 1}^\infty$ can be chosen using a similar strategy as \citep{langford2008epoch} as long as $t > \nicefrac{2K^2}{\eta^2}$, i.e., 
\begin{align}
n_t = \left\{
\begin{array}{lr}
0,     &  t \leq \nicefrac{2K^2}{\eta^2}; \\
\lfloor \sqrt{\nicefrac{t}{K \iota}} \rfloor,     & t > \nicefrac{2K^2}{\eta^2}.
\end{array}
\right.
\end{align}

By the similar argument as \citep{langford2008epoch}, we obtain the following regret bound,
\begin{align}
\regret(T) = \Ocal \left(\frac{K^{\nicefrac{1}{3}}T^{\nicefrac{2}{3}} \iota^{\nicefrac{1}{3}}}{\Delta\psi_{\pi^\star}} + \frac{2K^2 \iota}{\eta^2} \right). \tag*{\qedhere}
\end{align}
\end{proof}



\section{Details of Experiments}
\label{sec:expdetail}

\subsection{Mitigation Strategies of the Issues in Optimizing $\pi$ and $\psi$ Jointly}

In the practical implementation of {\algoname}, we observe two critical issues that could cause poor performance. i) When using the gradient-based method to optimize the softmax policy class, 
it exhibits sensitivity to parameter initialization\footnote{This is a well known difficulty of using gradient-based methods to optimize the softmax policy~\cite{mei2020escaping}.}. ii) For the objective function Eq.\eqref{eq:popobj}, $(\pi^\dagger, \psi^\dagger)$ is also a local minimum of it, where $\pi^\dagger \coloneqq \argmin_{\pi \in \Pi} V(\pi)$ and $\psi^\dagger \coloneqq \argmin_{\psi \in \Psi} \Delta\psi$. It is observed in practice that, optimizing Eq.\eqref{eq:popobj} could easily suffer from that local minimum and converge to the opposite optimization direction, $(\pi^\dagger, \psi^\dagger)$.


To address these two issues, we use the following two important components in our implementation:
\paragraph{Adaptive Restart Procedure}
Although there is no access to the explicit reward in the interaction-grounded learning, Eq.\eqref{eq:anapopobj} shows that the value of $V(\pi,\psi) - V(\pi_{\bad},\psi)$ can be used to measure the quality of $\pi$ and $\psi$, where $\pi_{\bad}$ is the uniform policy. Therefore, we use the following indicator to measure the performance of our learned policy $\pi$ empirically for both {\algoname},
\begin{align}
\label{eq:localminidct}
\frac{1}{|\Dcal|}\sum_{(x,a,y) \in \Dcal} K \pi(a|s) \psi(y) - \frac{1}{|\Dcal|}\sum_{(x,a,y) \in \Dcal} \psi(y),
\end{align}
This indicator can be viewed as the importance weighted estimator of $V(\pi,\psi) - V(\pi_{\bad},\psi)$ as we use the uniform policy as the data-gathering policy over our experiments. 
%
By Lemma \ref{lem:decodebadpi}, we know that the local minimum $(\pi^\dagger, \psi^\dagger)$ cannot have a large value on indicator Eq.\eqref{eq:localminidct} when we have enough data. In addition, we can use the value of Eq.\eqref{eq:localminidct} to detect if {\algoname} achieves a near-optimal solution by the following lemma.
\begin{lemma}
\label{lem:idctor}
Let $\pi_{\bad}$ be the uniform policy with $\Dcal$ obtained by $\pi_{\bad}$. We define $\Vhat_\Dcal(\pi,\psi) \coloneqq K \Eop_{(x,a,y) \sim \Dcal}\left[ \psi(y) \pi(a|s) \right]$ and  $\Vhat_\Dcal(\pi_{\bad},\psi) \coloneqq \Eop_{(x,a,y) \sim \Dcal}\left[ \psi(y) \right]$ for any $(\pi,\psi)$. Then, if $\Delta \psi > 0$\footnote{One sufficient condition of $\Delta \psi > 0$ is \Eqref{eq:pitildegvb_a}. In our experiments, we introduce a data-driven corrector for the reward decoder (stated as below). Together with the adaptive restart procedure, we could ensure $\Delta \psi > 0$ with much less data.},
we have
\begin{align}
V(\pi) \geq \frac{\Vhat_\Dcal(\pi,\psi) - \Vhat_\Dcal(\pi_{\bad},\psi) - K \varepsilon_{\Dcal}}{\Delta \psi^\star} + V(\pi_{\bad}).
\end{align}
\end{lemma}
\begin{proof}[\bf\em Proof of Lemma \ref{lem:idctor}]
\begin{align}
\Vhat_\Dcal(\pi,\psi) - \Vhat_\Dcal(\pi_u,\psi) \leq &~ V(\pi,\psi) - V(\pi_u,\psi) + K \varepsilon_\Dcal
\\
= &~ (V(\pi) - V(\pi_u))\Delta\psi +  K \varepsilon_\Dcal \tag{by \Eqref{eq:anapopobj}}
\\
\leq &~ (V(\pi) - V(\pi_u))\Delta\psi^\star +  K \varepsilon_\Dcal \tag{$\Delta\psi > 0$}
\\
\Longrightarrow V(\pi) \geq &~ \frac{\Vhat_\Dcal(\pi,\psi) - \Vhat_\Dcal(\pi_u,\psi) - K \varepsilon_{\Dcal}}{\Delta \psi^\star} + V(\pi_u).
\end{align}
This completes the proof.
\end{proof}
%
A direct consequence of Lemma \ref{lem:idctor} is that, if we have enough data, and the value of Eq.\eqref{eq:localminidct} is greater than some threshold close to $V(\pi_{\bad})$ (which is $1/K$ in our experiments), then the value of Eq.\eqref{eq:localminidct} actually controls the expected return of $\pi$.

Therefore, we propose an adaptive restart procedure that works in the following way: During the early training stage of {\algoname} and {\sf CB}, we set a threshold for {\sf CB} and {\algoname}. If the value of the above indicator does not go beyond that threshold, we restart the training of the corresponding algorithm until that threshold is surpassed.

\paragraph{Data-Driven Corrector for Reward Decoder}
Lemma \ref{lem:decodebadpi} and Lemma \ref{lem:idctor} suggest that the adaptive restart procedure based on the value of Eq.\eqref{eq:localminidct} is able to not only address the issue of converging to a bad local minimum (when having enough data) but also overcomes the sensitivity to parameter initialization. However, the amount of data sufficient to avoid bad local minimum by applying the adaptive restart procedure is usually too large to use in the online setting.

We notice that if we define the opposite reward decoder $\psitilde \coloneqq 1 - \psi, \forall \psi \in \Psi$, then $\psitilde^\dagger \coloneqq 1 - \argmin_{\psi \in \Psi} \Delta\psi = \argmax_{\psi \in 1 - \Psi} \Delta\psi$, where $1 - \Psi \coloneqq \{1 - \psi:\psi \in \Psi\}$. Since the $\Psi$ class we used in our experiment is the linear classifier with sigmoid activation function, $1 - \psi$ and $\psi$ will only differ in the sign before the parameters. Thus, we have $\Psi = 1 - \Psi$ by setting the parameter space to be $\RR^d$, and $\psitilde^\dagger = \argmax_{\psi \in \Psi} \Delta\psi = \psi^\star$.

Since optimizing Eq.\eqref{eq:popobj} will either maximize $\Delta\psi$ or minimize $\Delta\psi$, if we can determine the learned $\psi$ is actually converging to $\psi^\dagger = \argmin \Delta\psi$, choosing its opposite decoder $\psitilde^\dagger = 1 - \psi^\dagger$ could still provide us the desired decoder of $\argmax_{\psi \in \Psi} \Delta\psi$. By following this fact, we add an additional layer at the output of each $\psi$ as a data-driven reward decoder corrector, which works as follows:
\begin{enumerate}[(1),topsep=1pt,parsep=1pt,itemsep=1pt,partopsep=1pt]
\item Calculate $ \sign \coloneqq \frac{1}{|\Dcal|}\sum_{y \in \Dcal} \1 ( \psi(y) > 0.5)$.
\item If $\sign \leq 0.5$, output the original prediction $\psi(y), \forall y$.
\item If $\sign > 0.5$, output the corrected prediction $1 - \psi(y), \forall y$.
\end{enumerate}
Note that, the corrector above is designed based on data obtained from the uniform policy $\pi_{\bad}$. In this case, the number of reaction signals generated from $r = 0$ should be more than that of $r = 1$. Otherwise, that data-driven reward decoder corrector should be adjusted according to the data-gathering policy by using an importance weighted estimation.

\subsection{Additional Experimental Details}
We now provide some addition details in our implentation.
The experiments were conducted using Google Colab CPU instance, which was based on Intel Xeon CPU (2.30GHz) and 12 GB memory. No GPU was used. The prototype codes were built over {\tt Python} and {\tt PyTorch}. With the single process in the setup above, a single trial of either batch experiment or iterative experiment took less than 15 minutes to finish.
%
%
\paragraph{Environment}
The infinite MNIST environment used in Section \ref{sec:experiments} is built based on the {\tt infinite MNIST dataset} at \url{https://leon.bottou.org/projects/infimnist} and a {\tt Python} building of the infinite MNIST dataset generator at \url{https://github.com/albietz/infimnist_py}. To increase the speed of the experiment, we pre-generate a context set with $\sim 500000$ samples, and two sets of reaction vector, $\sim 100000$ samples of both image ``0'' and image ``1''. At each time step, the context is randomly selected from that pre-generated context set, and the reaction vector is randomly selected from the two sets of reaction vectors according to the actual reward.

The environment in Section \ref{sec:failunsup} is built similarly, and only differs in the way of generating the reaction vector sets. In ENV (1), we generated two reaction vector sets with size $\sim 100000$ for both, one is consist of the mixture of images of ``'0'' and ``1'', with the ratio of $8:1$. Another reaction vector set only includes images of ``2''. Therefore, if the actual reward is 0, a random image from the first set is selected as the reaction vector, otherwise, a random image from the second set will be selected.
Those of ENV (2) only exchange the positions of images of ``1'' and images of ``2'', and others are set similarly.

\paragraph{Implementation of $\Pi$ and $\Psi$}
Both of the policy class $\Pi$ and reward decoder $\Psi$ are using linear classifiers. The policy $\Pi$ is implemented using the regular softmax policy with the temperature of $1$, and the reward decoder class $\Psi$ uses the sigmoid prediction with temperature of $0.1$.

\paragraph{Details in the Iterative Algorithm}
The total number of round is $10000$ in our experiment (the ``round'' denotes ``$i$'' in step \ref{algostep:step_i} of Algorithm \ref{alg:algoepgreedy}, and each ``round'' may contain more than one interactions due to Algorithm \ref{alg:algoepgreedy}'s suggestion).
We set a number of $4000$ samples as the “warm-up” data for \algoname, which is suggested by Theorem \ref{thm:epoch_greedy}.
After that, we use the uniform policy for exploration and the exploitation scheduling is set based on the suggestion in Theorem \ref{thm:epoch_greedy}. That is, at each round $i$ ($i$ starts with $4001$), we act one-step uniform exploration and $\sqrt{\nicefrac{i}{100K}}$ steps exploitation. To accelerate the training process, we update the parameter every 100 rounds.







\end{document}